\RequirePackage[svgnames]{xcolor}

\documentclass[11pt, letterpaper, logo]{merged_style}

\usepackage[comma,authoryear,compress]{natbib}
\usepackage{algorithm}
\usepackage{algorithmic}
\usepackage{mathtools}
\usepackage{nicefrac}
\usepackage{subcaption}
\usepackage{wrapfig}
\usepackage{multirow}
\usepackage{makecell}
\usepackage{placeins}

\newcommand{\method}{DemoPSD}
\newcommand{\student}{\pi_\theta}

\newcommand{\EE}{\mathbb{E}}
\newcommand{\KL}{\mathrm{KL}}
\newcommand{\Ent}{\mathcal{H}}
\DeclareMathOperator*{\argmin}{arg\,min}

\bannerlogos{%
\hfill
  \includegraphics[height=11mm]{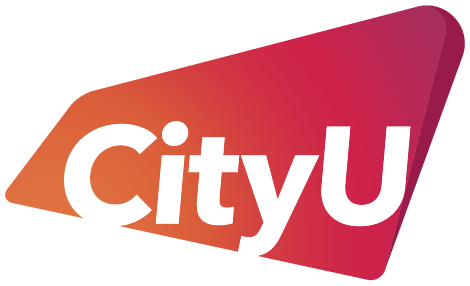}\hspace{10pt}%
  \includegraphics[height=12mm]{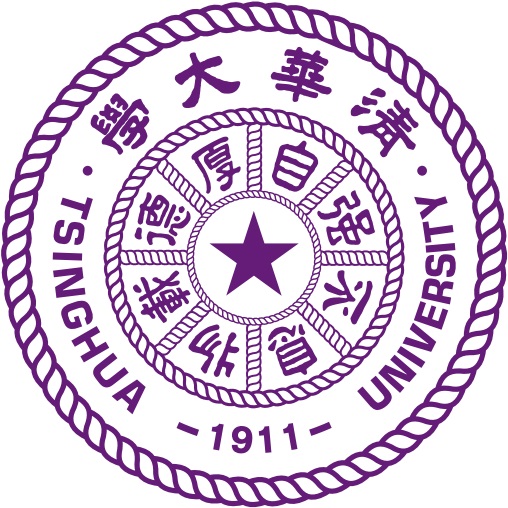}\hspace{10pt}%
  \includegraphics[height=12mm]{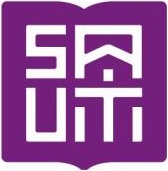}\hspace{10pt}%
  \includegraphics[height=12mm]{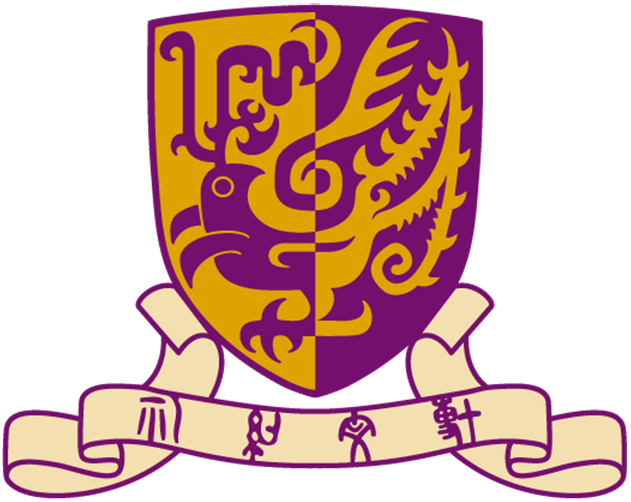}%
\vspace{-1.1cm}
}

\title{DemoPSD: Disagreement-Modulated Policy\\Self-Distillation}

\runningtitle{DemoPSD: Disagreement-Modulated Policy Self-Distillation}

\author{\begin{minipage}{\textwidth}\centering\fontsize{13}{14}\selectfont\bfseries Yunhe Li\textsuperscript{\normalfont *,1}, \quad Hao Shi\textsuperscript{\normalfont *,2}, \quad Wenhao Liu\textsuperscript{\normalfont 2}, \quad Mengzhe Ruan\textsuperscript{\normalfont 1}, \quad Hanxu Hou\textsuperscript{\normalfont 3}\par Zhongxiang Dai\textsuperscript{\normalfont 4}, \quad Shuang Qiu\textsuperscript{\normalfont \textdagger,1}, \quad Linqi Song\textsuperscript{\normalfont \textdagger,1}\par\vspace{5pt}{\fontsize{11}{13}\selectfont\normalfont\textsuperscript{\normalfont 1}City University of Hong Kong\quad \textsuperscript{\normalfont 2}Tsinghua University\par
\textsuperscript{\normalfont 3}Shenzhen University of Advanced Technology\quad \textsuperscript{\normalfont 4}Chinese University of Hong Kong, Shenzhen\par\vspace{5pt}\normalfont 
\texttt{uuen.li@my.cityu.edu.hk} \quad \texttt{shih22@mails.tsinghua.edu.cn} \quad \texttt{	\{shuanqiu,linqi.song\}@cityu.edu.hk}}
\end{minipage}}

\begin{abstract}
On-policy self-distillation (OPSD) has emerged as a practical method for training large language models (LLMs) to reason, where a single model acts as both the teacher and the student with different levels of information access. 
However, recent studies have found that the teacher's dense token-level supervision, conditioned on privileged information, can lead to overfitting to in-domain patterns, suppress exploration, and hurt cross-domain generalization, while also introducing a more fundamental issue: \textit{privileged information leakage}, where the student encodes answer-dependent shortcuts that are unavailable at test time.
We introduce \textbf{\method{}}, a novel framework that resolves such problems through the idea of \emph{selective adoption of teacher guidance}: the student adopts the teacher's guidance when their distributions remain reasonably consistent, and relies more on its own reasoning when their distributions substantially diverge, indicating that the teacher's output is overly influenced by privileged information. Instead of fitting the full teacher distribution, \method{} steers the student toward a \emph{reverse-KL barycenter target}, a weighted geometric combination of the teacher and student distributions, that naturally balances learning from the teacher with preserving the student's own reasoning capacity. We measure the difference between their distributions and use such a discrepancy to adaptively control the blending at each token position.
We provably show that \method{} achieves \textbf{(1)}~\emph{leakage attenuation}, i.e., effective mitigation of privileged information leakage; and \textbf{(2)}~\emph{exploration preservation}, i.e., preservation of exploration capacity under dense token-level distillation. Extensive experiments on SciKnowEval across four scientific fields show that \method{} outperforms both GRPO and SDPO  while maintaining higher training entropy and robustly generalizing to out-of-distribution GPQA benchmarks.
\end{abstract}

\begin{document}

\maketitle
\begingroup
\renewcommand{\thefootnote}{}\NoHyper\footnotetext{\begin{tabular}[t]{@{}l@{}}\textsuperscript{*}Equal contribution.\\
\textsuperscript{\textdagger}Corresponding author.\end{tabular}}\endNoHyper
\endgroup

\begin{figure}[t]
    \centering
    \begin{subfigure}[t]{0.48\textwidth}
        \centering
        \includegraphics[width=\textwidth]{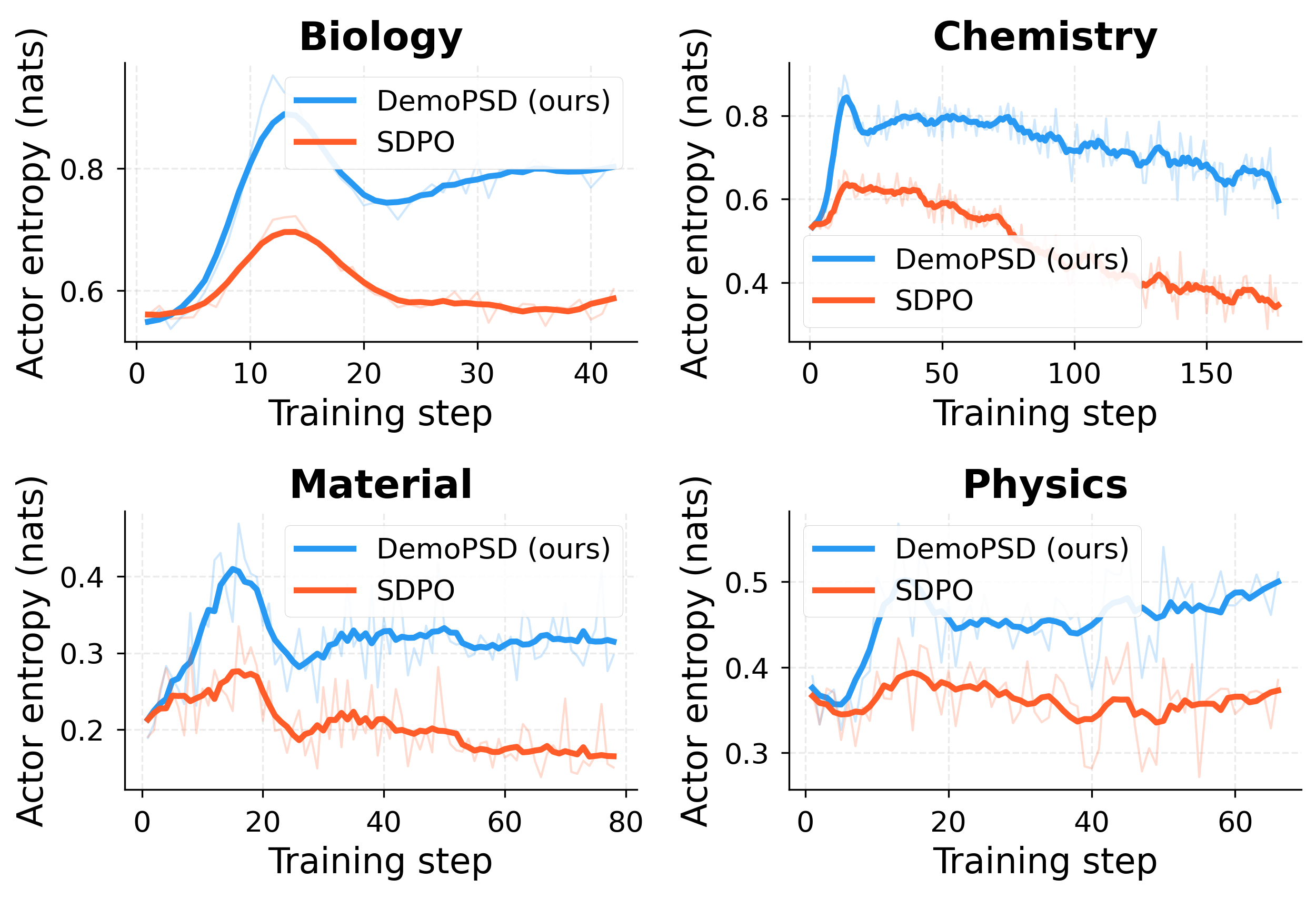}
        \caption{{Policy entropy over training steps.} \method{} maintains 33-98\% higher entropy than SDPO across all domains, avoiding policy entropy collapse.}
        \label{fig:entropy-curves}
    \end{subfigure}
    \hfill
    \begin{subfigure}[t]{0.49\textwidth}
        \centering
        \includegraphics[width=\textwidth, height=0.257\textheight]{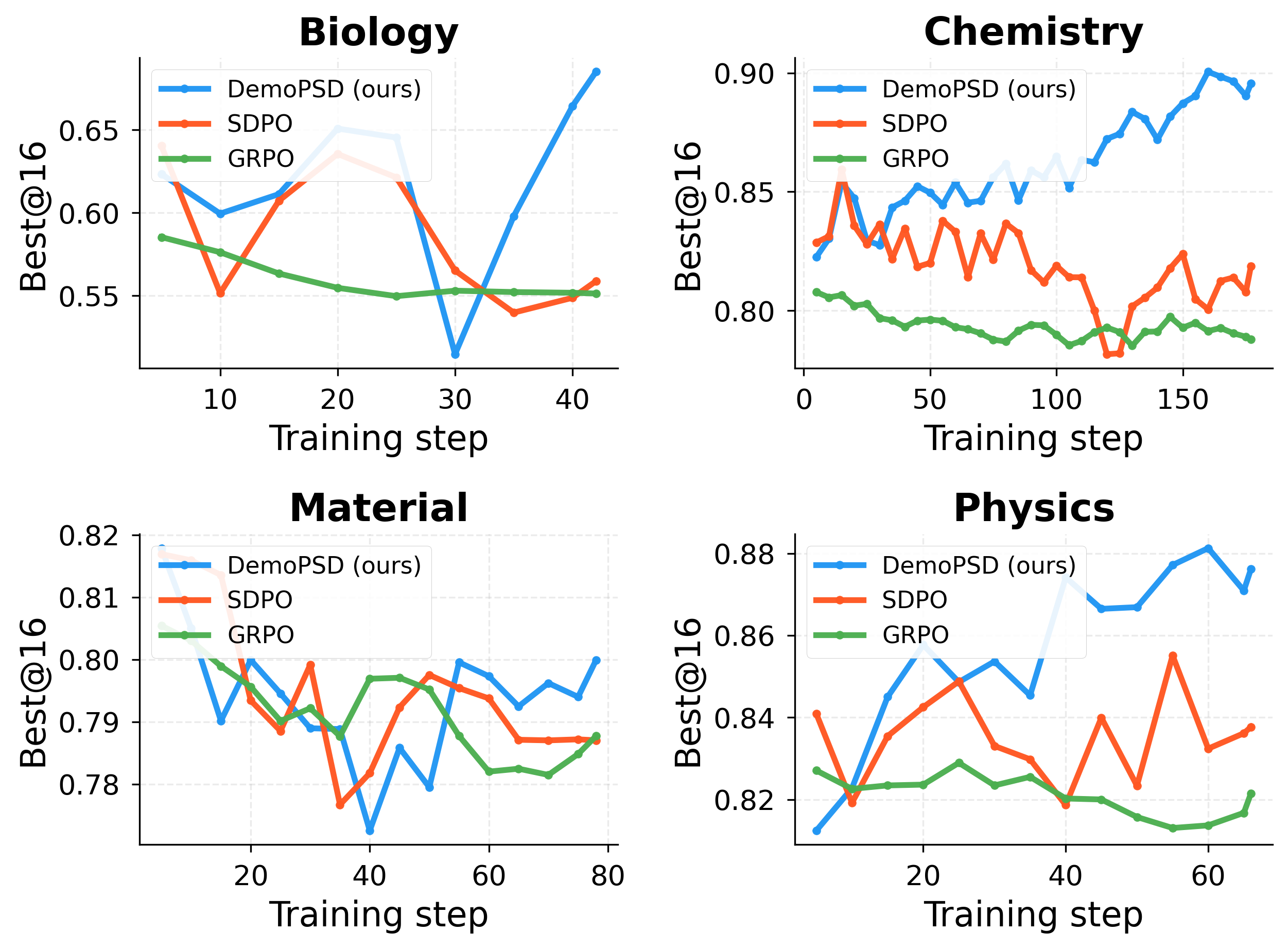}
        \caption{{Best@16 for each SciKnowEval  domain over training steps.}}
        \label{fig:best16-curves}
    \end{subfigure}
    \caption{\method{} preserves higher entropy (left), which translates into better best@16 performance (right).}
    \label{fig:highlight}
\end{figure}

\section{Introduction}
\label{sec:intro}

Reinforcement learning with verifiable rewards (RLVR) has become a central paradigm for post-training large language models on reasoning tasks~\citep{shao2024deepseekmath,deepseek2025r1,yu2026dapo}. Methods such as Group Relative Policy Optimization (GRPO) train models by sampling multiple rollouts per question and using outcome correctness as a reward signal. While effective, RLVR suffers from a fundamental \emph{credit assignment bottleneck}: standard RLVR methods distribute a rollout-level reward uniformly among all tokens in a rollout, offering coarse token-level credit signals that fail to distinguish individual token contributions~\citep{hubotter2026sdpo}.

On-policy distillation (OPD) addresses this bottleneck by introducing dense, token-level supervision from a teacher model on the student's self-generated trajectories~\citep{agarwal2024policy,gu2024minillm,lu2025opd}. Unlike off-policy distillation, which trains on teacher-generated texts but suffers from exposure bias~\citep{ross2011dagger,song2026survey}, OPD allows the student to learn from its own distribution while receiving dense teacher supervision. This paradigm has been widely adopted in industry, including Qwen3~\citep{qwen3} and DeepSeek-V4~\citep{deepseekv4}, establishing OPD as a practical complement to RLVR. The broader idea of transferring the knowledge of one neural network into another has earlier roots in \citet{schmidhuber1991neural,schmidhuber1992learning}, which proposed collapsing a self-organizing predictor hierarchy into a single recurrent network.

A particularly appealing variant is on-policy self-distillation (OPSD) ~\citep[e.g.,][]{zhao2026opsd,hubotter2026sdpo}, where a single model serves as both teacher and student. The teacher is the same model conditioned on privileged information, such as a verified reasoning trace or ground-truth answer, while the student receives only the question. OPSD eliminates the need for an external teacher and has demonstrated several-fold improvements in token efficiency over GRPO~\citep{zhao2026opsd,shenfeld2026self}. However, recent theoretical and empirical analysis has revealed a critical failure mode: \emph{privileged information leakage}~\citep{yang2026rlsd}. 
Because the teacher conditions on privileged information $y^*$ that the student never observes at test time, the OPSD objective contains an irreducible mutual information gap $I(y_t; y^* \mid x, y_{<t}) > 0$, which is a conditional mutual information with $x$ the input question, $y_{<t}$ the generated prefix, $y_t$ the next token, and $y^*$ the privileged information available only to the teacher. 
A positive value indicates that, even after conditioning on the question and generated prefix, the privileged signal still provides additional information about the next token, drving the student to encode answer-dependent shortcuts. This manifests as early performance gains followed by gradual degradation. As a result, the student may internalize cues tied to the privileged information instead of acquiring transferable reasoning strategies~\citep{yang2026rlsd}.

This failure mode reflects a broader tension between benefiting from the teacher's guidance and preserving the student's ability to reason independently.
Privileged information leakage is a symptom of a more fundamental design choice: standard OPSD optimizes the student to imitate the teacher's privileged conditional distribution at every token. This objective is problematic for two major reasons: first, the teacher's distribution at certain positions (e.g., numerical answers, solution-revealing steps) is shaped by privileged information rather than transferable reasoning; second, directly matching the teacher's distribution at every token can suppress the student's own reasoning capacity when privileged information is unavailable.

Several recent studies have proposed different mechanisms to address this challenge. RLSD~\citep{yang2026rlsd} avoids leakage entirely by using self-distillation only for magnitude estimation. HDPO~\citep{ding2026hdpo} restricts privileged distillation to ``cliff prompts.'' EGRSD~\citep{ke2026egrsd} gates distillation by teacher entropy. SRPO~\citep{li2026srpo} routes entire samples between GRPO and SDPO based on rollout correctness. DASD~\citep{zhang2026dasd} uses entropy-routed directional supervision, pulling low-entropy tokens toward the privileged teacher while pushing high-entropy tokens away to preserve exploration. GATES~\citep{stein2026gates} uses consensus among multiple teachers. The aforementioned methods share a common intuition: \emph{not all tokens are equally trustworthy}. Yet they all rely on \emph{indirect} proxies such as the teacher's entropy, sample correctness, student entropy, or multi-teacher consensus rather than directly measuring \emph{how much the teacher's prediction is influenced by privileged information}. How to design the distillation target itself to balance teacher-guided learning with the student's own reasoning, however, has received relatively little attention.

Our work introduces \method{}, \textbf{\underline{D}}isagr\textbf{\underline{e}}ement-\textbf{\underline{m}}\textbf{\underline{o}}dulated \textbf{\underline{P}}olicy \textbf{\underline{S}}elf-\textbf{\underline{D}}istillation, a novel framework that addresses this challenge in standard OPSD through the principle of \emph{selective adoption of teacher guidance}: the student adopts the teacher's guidance when their distributions are reasonably consistent, and relies more on its own reasoning when the teacher's distribution substantially diverges from the student's, indicating that the teacher’s output is overly influenced by privileged information. Rather than fitting the full teacher distribution, \method{} trains the student toward a \emph{reverse-KL barycenter target}, which is a weighted geometric combination of the teacher's and student's distribution:
\begin{align}
\begin{aligned}\label{eq:intro-target1}
    \pi_t^{\text{target}}(v \mid x, y^*, \hat{y}_{<t}) &\propto \big(\pi_{\text{teacher}}(v \mid x, y^*, \hat{y}_{<t} )\big)^{1-\alpha_t} \cdot \big(\pi_{\text{student}}(v \mid x, \hat{y}_{<t})\big)^{\alpha_t},
\end{aligned}
\end{align}
where $\alpha_t$ is a per-token \textit{leakage attenuation coefficient} determined by the disagreement between the distributions of the teacher and the student, controlling how far the target is interpolated from the teacher's distribution toward the student's distribution. When $\alpha_t$ is sufficiently small, the privileged information $y^*$ does not substantially shift the teacher's distribution, the target therefore remains close to the teacher. As $\alpha_t$ increases, the teacher's distribution becomes more strongly shaped by $y^*$. Forcing the student to directly match the teacher would encode answer-dependent shortcuts into the student, which is precisely the privileged information leakage. The target in \eqref{eq:intro-target1} is therefore interpolated further toward the student's distribution to attenuate leakage while preserving the student's unprivileged reasoning capacity. Figure~\ref{fig:highlight} previews our main empirical results based on the principle of selective adoption of teacher guidance. \method{} preserves substantially higher training entropy 
than SDPO across all domains, which further translates into improved best@16 results.

\noindent\textbf{Our contribution.} Specifically, our main contributions are three-fold:

\begin{enumerate}[leftmargin=*,itemsep=3pt]
\item We propose a novel on-policy self-distillation algorithm \method{} that effectively prevents the student model from overfitting the teacher's distribution, thereby improving both in-domain and cross-domain reasoning capabilities and reducing privileged information leakage during self-distillation.

    \item We theoretically prove two key properties of \method{}: \textbf{(1)} \emph{leakage attenuation}, i.e., the disagreement-weighted reverse-KL barycenter target reduces the rate of privileged information leakage; and \textbf{(2)} \emph{exploration preservation}, i.e., \method{} retains the student's exploration capacity under dense token-level distillation.
    \item We conduct extensive experiments on SciKnowEval across four scientific domains. Our empirical results show that \method{} consistently outperforms SDPO and GRPO, achieving up to a 4.2\% improvement in @16 accuracy while maintaining 35--97\% higher training entropy. On the out-of-distribution GPQA benchmark, \method{} maintains strong overall accuracy and demonstrates robust generalization, whereas SDPO exhibits a gradual performance decline.
\end{enumerate}

\section{Related Work}
\label{sec:related}

\noindent\textbf{On-Policy Distillation and Self-Distillation.}
Recent OPD methods such as GKD~\citep{agarwal2024policy} and MiniLLM~\citep{gu2024minillm} train the student on model-generated trajectories while using teacher distributions as dense supervision. This on-policy design is motivated by the classic imitation-learning observation that training only on expert-generated states can suffer from compounding errors under distribution shift~\citep{ross2011dagger}. Subsequent work further studies OPD from different perspectives: REOPOLD~\citep{ko2026reopold} relaxes on-policy distillation for more efficient reasoning, Veto~\citep{jang2026veto} reformulates the distillation target to improve training stability, and \citet{song2026opdsurvey} provide a broader survey of OPD methods for large language models.
Another line of research focuses on the on-policy self-distillation problem. Self-Distilled Reasoner~\citep{zhao2026opsd} studies the setting where a single model provides its own on-policy distillation signal for reasoning. SDPO~\citep{hubotter2026sdpo} further frames reinforcement learning through self-distillation, converting sparse outcome feedback into dense training signals. Related variants explore complementary design choices: SD-Zero~\citep{he2026sdzero} uses self-revision to transform binary rewards into dense supervision, UniSD~\citep{jin2026unisd} proposes a unified framework for self-distillation in LLMs, and CRISP~\citep{sang2026crisp} applies iterative self-policy distillation to compressed reasoning. As shown in \S\ref{subsec:loss-gradient}, \method{} instead uses a disagreement-dependent geometric target that preserves dense supervision on low-disagreement tokens while attenuating teacher-induced signals on high-disagreement tokens.

\vspace{3pt}
\noindent\textbf{Addressing Privileged Information Leakage}
Recent work has begun to analyze and mitigate failure modes in on-policy self-distillation. \citet{yang2026rlsd} study self-distilled RLVR and identify privileged information leakage as a key concern. HDPO~\citep{ding2026hdpo} focuses privileged self-distillation on cliff prompts, while PBSD~\citep{yu2026pbsd} moves beyond direct KL matching through preference-based self-distillation and reward regularization.
Other methods adjust when or how self-distillation is applied. SRPO~\citep{li2026srpo} unifies group-relative optimization and self-distillation through sample routing, DASD~\citep{zhang2026dasd} adapts supervision according to the direction of the self-distillation signal, and PAINT~\citep{tan2026paint} interpolates between partial- and full-solution prompts. \citet{kim2026whydegrade} analyzes why self-distillation can degrade reasoning ability. In contrast, \method{} keeps the token-level distillation setting but changes the distributional target itself: the reverse-KL barycenter adaptively interpolates between the privileged teacher and the unprivileged student according to teacher-student disagreement.

\vspace{3pt}
\noindent\textbf{Mixture Distributions and Entropy Dynamics}
AMiD~\citep{shin2025amid} introduces $\alpha$-mixture assistant distributions for knowledge distillation, making it conceptually related to our reverse-KL barycenter target, although AMiD does not address privileged self-distillation. Entropy stability has also emerged as an important issue in large-scale RL training systems such as DAPO~\citep{yu2026dapo} and in explicit entropy-control methods such as EntroPIC~\citep{yang2025entropic}, motivating our focus on preserving exploration during dense distillation. PACED~\citep{xu2026paced} studies distillation and on-policy self-distillation at the frontier of student competence, which is complementary to our token-level disagreement-based target adaptation. 

\section{Background and Problem Setting}
\label{sec:background}

\subsection{Reinforcement Learning with Verifiable Rewards}
\label{subsec:rlvr}

We consider the standard RLVR setup for post-training LLMs. Given a dataset of questions $\mathcal{D} = \{(x_i, a_i^*)\}_{i=1}^N$ where $a_i^*$ is the verifiable answer, the model $\student(\cdot \mid x)$ generates rollouts $y \sim \student(\cdot \mid x)$ and receives a binary reward $r(y, a^*) \in \{0, 1\}$ based on outcome correctness. GRPO~\citep{shao2024deepseekmath} estimates advantages from these rewards within each rollout group and optimizes:
\begin{equation}
\label{eq:grpo}
\mathcal{L}_{\text{GRPO}}(\theta) = -\EE_{x \sim \mathcal{D}} \EE_{y \sim \student(\cdot|x)} \left[ \hat{A}(y) \cdot \log \student(y|x) \right] + \beta_{\text{KL}} \cdot \KL(\student \| \pi_{\text{ref}}),
\end{equation}
where $\hat{A}(y)$ is the group-relative advantage. For a group of $G$ rollouts $\{y_j\}_{j=1}^G$ sampled for the same question $x$, GRPO computes
\begin{equation}
\label{eq:grpo-advantage}
\hat{A}(y_j) = \frac{r(y_j, a^*) - \mu_r}{\sigma_r + \epsilon}, \qquad
\mu_r = \frac{1}{G}\sum_{k=1}^G r(y_k, a^*), \quad
\sigma_r = \sqrt{\frac{1}{G}\sum_{k=1}^G \big(r(y_k, a^*) - \mu_r\big)^2},
\end{equation}
with a small constant $\epsilon>0$ for numerical stability. The KL regularizer is defined as
\begin{equation}
\label{eq:grpo-kl}
\KL(\student \| \pi_{\text{ref}})
= \EE_{x \sim \mathcal{D}}\!\left[\sum_y \student(y \mid x)\log\frac{\student(y \mid x)}{\pi_{\text{ref}}(y \mid x)}\right]
= \EE_{x \sim \mathcal{D},\, y \sim \student}\!\left[\sum_{t=1}^{|y|}\log\frac{\student(y_t \mid x,y_{<t})}{\pi_{\text{ref}}(y_t \mid x,y_{<t})}\right].
\end{equation}
One of the fundamental limitations is that $r$ provides only an outcome reward per response, offering no guidance on which tokens contributed more to success or failure.

\subsection{On-Policy Self-Distillation}
\label{subsec:opsd}

Reinforcement learning via self-distillation (SDPO)~\citep{hubotter2026sdpo} addresses the credit assignment bottleneck by introducing dense, token-level supervision from a privileged version of the same model. The teacher $\pi_{\theta}(\cdot \mid x, y^*)$ is the current model conditioned on both the question $x$ and privileged information $y^*$ (e.g., a verified reasoning trace or ground truth), while the student $\student(\cdot \mid x)$ receives only the question. Given a student-generated rollout $\hat{y} \sim \student(\cdot \mid x)$, the SDPO objective minimizes per-token divergence:
\begin{equation}
\label{eq:opsd}
\mathcal{L}_{\text{SDPO}}(\theta) = \EE_{x \sim \mathcal{D}} \EE_{\hat{y} \sim \student(\cdot|x)} \left[ \sum_{t=1}^{|\hat{y}|} \KL\big(\student(\cdot \mid x, \hat{y}_{<t}) \,\|\, \text{stopgrad}(\pi_\theta(\cdot \mid x, y^*, \hat{y}_{<t}))\big) \right].
\end{equation}
The key insight is that the teacher leverages its  access to the privileged $y^*$ to provide richer feedback than an outcome reward. The stopgrad operator prevents gradients from flowing into the teacher, which keeps the teacher from shifting toward the student and  ignoring $y^*$.

\subsection{The Privileged Information Leakage Problem}
\label{subsec:leakage-problem}

While SDPO achieves impressive token efficiency, \citet{yang2026rlsd} proved that the setting is fundamentally ill-posed. Since the teacher conditions on privileged information $y^*$ that the student cannot observe, the SDPO objective contains an irreducible mutual information gap:
\begin{equation}
\label{eq:mi-gap}
I(y_t; y^* \mid x, y_{<t}) > 0.
\end{equation}
This gap implies the student can never perfectly achieve the teacher's conditional distribution, regardless of capacity. At the gradient level,  per-sample gradients include an $y^*$-specific deviation whose variance is proportional to this mutual information. At the early stage of training, the beneficial gradient component dominates, producing rapid training reward improvement. However, as the student approaches the teacher's marginal distribution, the deviation takes over, driving the student to encode $x \to y^*$ correlations, which is exactly the privileged information leakage explained in ~\citet{yang2026rlsd}. Empirically, SDPO performance peaks early and then gradually degrades during the remaining training stage.

The leakage problem points to a deeper issue: the teacher's distribution is not always an appropriate target for direct fitting. Even if leakage could be eliminated, a student who exactly replicates the teacher has lost its own capacity for reasoning. What we need instead is a training target that adaptively incorporates the teacher's guidance while preserving the student's own reasoning ability.

\section{The Proposed Method: \method{}}
\label{sec:method}

This section presents the proposed method \method{}, built on the principle of selective adoption of teacher guidance, i.e., the student follows the teacher's guidance when privileged information does not heavily distort the teacher's distribution so that it diverges substantially from the student's distribution. Below, we first describe how to measure teacher-student disagreement (\S\ref{subsec:contrast}), then introduce the reverse-KL barycenter target that implements selective adoption (\S\ref{subsec:geometric}), derive its loss and gradient (\S\ref{subsec:loss-gradient}), and describe the full training procedure (\S\ref{subsec:training}).


\subsection{Measuring Teacher-Student Disagreement}
\label{subsec:contrast}

The key ingredient of \method{} is measuring the \emph{disagreement} between the teacher's and student's predictions at each token position: one made \emph{with} privileged information, and one made \emph{without}. Token positions where these predictions remain reasonably consistent are likely to reflect transferrable knowledge that the student can safely adopt, while positions where they substantially disagree indicate that the teacher's output has been overly influenced by privileged information.





\vspace{3pt} \noindent\textbf{Disagreement and Leakage Attenuation Coefficient.} At each token position $t$, the privileged teacher's prediction is obtained by conditioning the model on the question $x$, the privileged information $y^*$, and the student's rollout prefix $\hat{y}_{<t}$. For notational convenience, we write this distribution as $\pi_T^t(v, y^*)$ as shorthand for $\pi_\theta(v \mid x, y^*, \hat{y}_{<t})$. The corresponding student's prediction conditions only on $x$ and $\hat{y}_{<t}$, and we write it as $\pi_S^t(v)$ as shorthand for $\pi_\theta(v \mid x, \hat{y}_{<t})$. We use these shorthand notations when no ambiguity arises and revert to the full conditional form when the conditioning context needs to be made explicit. The privileged prediction provides a rich teacher signal because it receives $y^*$, while the student's prediction serves as the reference distribution for evaluating disagreement. We describe how $y^*$ is incorporated into the model's context in \S\ref{subsec:training}. In practice, for training stability, we use a separate exponential moving average (EMA) copy of the student when computing the disagreement in \eqref{eq:js-divergence} and the target distribution in \eqref{eq:target-unnorm}; implementation details are summarized in Algorithm~\ref{alg:dupsd}. We measure the disagreement $d_t$ between $\pi_T$ and $\pi_S$ by using the Jensen-Shannon divergence (JSD):
\begin{equation}
\label{eq:js-divergence}
d_t = \mathrm{JSD}(\pi_S^t \| \pi_T^t) = \frac{1}{2}\KL(\pi_S^t \| m_t) + \frac{1}{2}\KL(\pi_T^t \| m_t), \quad m_t = \frac{1}{2}(\pi_S^t + \pi_T^t).
\end{equation}
From $d_t$, we derive a leakage attenuation coefficient $\alpha_t = f(d_t)$ that controls how much the target shifts away from the privileged teacher and toward the student's own prediction. We require $f$ to be monotonically increasing so that larger teacher-student disagreement leads to stronger leakage attenuation, and to satisfy $f(0)=0$ so that the target reduces to the teacher distribution when the two predictions match. We also use a saturating form with $\lim_{d \to \infty} f(d)=\alpha_{\max}$, which prevents extreme disagreement from completely discarding the teacher signal. The cap $\alpha_{\max}$ is an empirical hyperparameter: setting it too large assigns too little weight to the teacher distribution and can weaken useful distillation signals. We realize $f$ via a rescaled sigmoid:
\begin{equation}
\label{eq:alpha-remap}
\alpha_t = \big(\sigma(\beta \cdot d_t) - 0.5\big) \cdot 2 \cdot \alpha_{\max},
\end{equation}
where $\beta$ controls the sensitivity of the gate to teacher-student disagreement: a larger $\beta$ makes $\alpha_t$ increase more sharply with small changes in $d_t$, causing the target to move away from the privileged teacher more aggressively, whereas a smaller $\beta$ yields a smoother transition and retains more teacher signal under moderate disagreement. This realization has two key properties: \textbf{(1)} When $\alpha_t$ is sufficiently small, i.e.,  the two distributions are reasonably consistent, it is safe to distill; \textbf{(2)} As $\alpha_t$ increases to $\alpha_{\max}$, i.e., they strongly disagree, distillation becomes increasingly risky.

\subsection{Reverse-KL Barycenter Target}
\label{subsec:geometric}

Given the coefficient $\alpha_t$, we define the distillation target as a geometric mixture of the two distributions. The  target at token position $t$ is:
\begin{equation}
\label{eq:target-unnorm}
\pi_{\text{target}}^{\alpha_t}(v \mid x, y^*, \hat{y}_{<t}) \propto \big(\pi_T^t(v, y^*)\big)^{1-\alpha_t} \cdot \big(\pi_S^t(v)\big)^{\alpha_t}.
\end{equation}
This distribution is the reverse-KL barycenter of the privileged teacher and the student distributions under the weight $\alpha_t$, defined by
\begin{equation}
\label{eq:reverse-kl-barycenter}
\pi_{\text{target}}^{\alpha_t}
= \argmin_{q \in \Delta(\mathcal{V})}
\left\{
(1-\alpha_t)\KL\!\left(q \middle\| \pi_T^t\right)
+ \alpha_t\KL\!\left(q \middle\| \pi_S^t\right)
\right\},
\end{equation}
where \(\Delta(\mathcal{V})\) denotes the probability simplex over the vocabulary $\mathcal{V}$. The reverse-KL barycenter in \eqref{eq:reverse-kl-barycenter} defines the weighted centroid of a collection of probability distributions, i.e., $\pi_T^t$ and $\pi_S^t$ in this problem, under the reverse KL divergence. Equivalently, this target interpolates between the teacher and student distributions in log-probability space,
\begin{align*}
        \log \pi_{\text{target}}^{\alpha_t}(v \mid x, y^*, \hat{y}_{<t})  =   (1-\alpha_t)\log \pi_{T}^{t}(v, y^*) +  \alpha_{t} \log \pi_{S}^t(v) - \log Z_{\alpha_t}, 
\end{align*}
where $Z_{\alpha_t}$ is the normalization term for \eqref{eq:target-unnorm} defined as:
\begin{equation}
\label{eq:target}
 Z_{\alpha_t} = \sum_v \pi_{\text{target}}^{\alpha_t}(v \mid x, y^*, \hat{y}_{<t}).
\end{equation}

\vspace{3pt} 
\noindent\textbf{Geometric Mixture vs Arithmetic Mixture.} The geometric mixture is chosen over the arithmetic alternative $((1-\alpha_t)\pi_T^t + \alpha_t\pi_S^t)$ for two reasons: 
\textbf{(1)} Because probabilities are multiplied, a token receives substantial target mass only when it is supported by both the privileged teacher and the student. Thus, tokens endorsed primarily by the teacher but assigned very low probabilities by the student are naturally suppressed, whereas an arithmetic mixture would still allocate them non-trivial mass.
\textbf{(2)} When the teacher and student distributions have different modes, an arithmetic mixture can average the modes into a diffuse target with inflated entropy. The geometric mixture avoids this mode-averaging effect, yielding a sharper and more coherent training signal. This is consistent with AMiD's~\citep{shin2025amid} observation that mixture geometry controls mode-covering versus mode-seeking behavior.

\subsection{Loss Function}
\label{subsec:loss-gradient}

The student is trained to minimize the reverse KL divergence objective toward the reverse-KL barycenter target:
\begin{equation}
\label{eq:dupsd-loss}
\mathcal{L}_{\text{DemoPSD}}(\theta) = \EE_{x \sim \mathcal{D}} \EE_{\hat{y} \sim \student(\cdot|x)} \left[ \sum_{t=1}^{|\hat{y}|} \KL\big(\pi_{\theta}(\cdot \mid x, \hat{y}_{<t}) \,\|\, \text{stopgrad}(\pi_{\text{target}}^{\alpha_t}(\cdot \mid x, y^*, \hat{y}_{<t}) )\big) \right].
\end{equation}
 Directly computing and differentiating through the normalization term $Z_{\alpha_t}$ would make the optimization complicated. However, the full target distribution is wrapped with stop-gradient: the teacher $\pi_T^t$, the reference student $\pi_S^t$, the weight $\alpha_t$ are all treated as fixed during the backward pass. Consequently, $Z_{\alpha_t}$  becomes  constant and the optimization hence avoids directly backpropagating through it.
Then the gradient of $\mathcal{L}_{\text{DemoPSD}}$ takes the following form:
\begin{equation}
\label{eq:grad-simple}
\nabla_\theta \mathcal{L}_{\text{DemoPSD}} = \EE_{\hat{y} \sim \student(\cdot|x)} \sum_{t=1}^{|\hat{y}|} \Big[  \EE_{\hat{y}_t \sim {\student}(\cdot \mid x, \hat{y}_{<t}) } {\color{red} (1-\alpha_t)}\ \log \frac{{\student}(\hat{y}_t \mid x, \hat{y}_{<t})}{\pi_\theta(\hat{y}_t \mid x, y^*, \hat{y}_{<t})  }\nabla_\theta \log {\student}(\hat{y}_t \mid x, \hat{y}_{<t}) \Big].
\end{equation}
 The \method{} gradient keeps the same reverse-KL score-function form while scaling the teacher-induced log-ratio signal by the disagreement-based factor $(1-\alpha_t)$. As illustrated in \eqref{eq:grad-simple}, positions with larger teacher-student disagreement contribute a weaker distillation signal, reducing the tendency to backpropagate privileged information dependent guidance from the teacher.

\subsection{Privileged Information Injection and Training Procedure}
\label{subsec:training}

Algorithm~\ref{alg:dupsd} summarizes the full \method{} algorithm.

\vspace{3pt} \noindent\textbf{Privileged Information Injection.}  Generally, for each training prompt $x$ with privileged information $y^*$, and a relevant student-generated rollout $\hat{y}$, we construct the  teacher's input by prepending $y^*$ to the prompt context:
$$[\texttt{Question:}~x~|~\texttt{Privileged Information:}~y^*~|~\texttt{Student Response:}~\hat{y}_{<t}].$$
The student model receives only:
$$[\texttt{Question:}~x~|~\texttt{Student Response:}~\hat{y}_{<t}],$$
both of which share the same model. The only difference is whether the privileged information $y^*$ is included in the conditioning context. 

\vspace{3pt} \noindent\textbf{Reprompting Mechanism.} For a correct rollout, the generated response itself contains rich solution information and can therefore serve as privileged information for the teacher model. Following \citet{hubotter2026sdpo}, we use a reprompting mechanism to construct this privileged context: for each prompt group, if at least one rollout is correct, we randomly select one correct rollout as $y^*$ and insert it into the teacher context above; if no rollout is correct, no reliable privileged teacher context can be formed, so the prompt is skipped for distillation. As explained in \citet{hubotter2026sdpo}, model performance is not sensitive to syntactic variations of the reprompting template, so we use a similar template to instantiate the privileged information for the teacher model.

\begin{algorithm}[t]
\caption{\method{}: \textbf{\underline{D}}isagr\textbf{\underline{e}}ement-\textbf{\underline{m}}\textbf{\underline{o}}dulated \textbf{\underline{P}}olicy \textbf{\underline{S}}elf-\textbf{\underline{D}}istillation}
\label{alg:dupsd}
\begin{algorithmic}[1]
\REQUIRE Dataset $\mathcal{D}$, student model $\student$, 
sensitivity $\beta$, max leakage attenuation coefficient $\alpha_{\max}$
\FOR{each training step}
    \STATE Sample batch $\{x_i\}$ from $\mathcal{D}$; generate rollouts $\hat{y}_i \sim \student(\cdot \mid x_i)$
    \STATE Filter: keep only prompts with at least one correct rollout ($r(\hat{y}_i, a_i^*) = 1$)
    \FOR{each filtered prompt $(x, y^*, \hat{y})$}
        \STATE Obtain the teacher policy: $\pi_T^t \leftarrow \pi_\theta(\cdot \mid x, y^*, \hat{y}_{<t})$ for all $t$
        \STATE Obtain the student policy: $\pi_S^t \leftarrow \pi_\theta(\cdot \mid x, \hat{y}_{<t})$ for all $t$
        \STATE {\color{blue} Compute the distributional disagreement  $d_t$ via \eqref{eq:js-divergence}}
        \STATE {\color{blue} Compute the leakage attenuation coefficient  $\alpha_t$ via \eqref{eq:alpha-remap}}
        \STATE {\color{blue} Compute the reverse-KL barycenter target via \eqref{eq:target-unnorm}}
    \ENDFOR
    \STATE Update $\theta$ via gradient descent on $\mathcal{L}_{\text{DemoPSD}}(\theta)$
\ENDFOR
\end{algorithmic}
\end{algorithm}


\section{Theoretical Analysis}
\label{sec:theory}

This work aims to solve a central question that how we preserve the token-level distributional supervision while suppressing privileged information leakage caused by conditioning the teacher on $y^*$? Standard OPSD exploits dense teacher distributions but is vulnerable to leakage.
In this section, we provide a detailed theoretical analysis of \method{} from the perspectives of leakage attenuation and exploration preservation. 

As we have analyzed in \S\ref{subsec:contrast}, in practice, we maintain a separate EMA copy of the student as the unprivileged reference, and construct the teacher based on the EMA copy for stability. Let $\pi_{\bar{\theta}}$ denote this EMA copy of the current student $\pi_\theta$. Throughout this section, both the privileged teacher distribution $\pi_T$ and the student distribution $\pi_S$ in the target are obtained by conditioning $\pi_{\bar{\theta}}$ on the corresponding privileged or unprivileged contexts. Following \citet{yang2026rlsd}, we define the per-step \emph{leakage rate} as the expected squared magnitude of the privileged deviation:
\begin{equation}
\label{eq:leakage-rate}
\mathcal{R}_{\text{leak}} = \EE_t\!\left[\|\Delta_t\|^2\right], \quad \text{where } \Delta_t(v) = \log \pi_T^t(v, y^*) - \log \pi_S^t(v).
\end{equation}
In this definition, $\Delta_t \in \mathbb{R}^{|\mathcal{V}|}$ is a vector indexed by tokens in the vocabulary $\mathcal{V}$, and $\|\Delta_t\|^2 = \sum_{v \in \mathcal{V}} \Delta_t(v)^2$ is its squared $\ell_2$ norm, measuring the total squared log-probability shift induced by $y^*$ at position $t$. Throughout this section, we write $\pi_\theta^t(v) := \pi_\theta(v \mid x, \hat{y}_{<t})$ for the student distribution at position $t$, and $\pi_{\text{target}}^{\alpha_t} := \pi_{\text{target}}^{\alpha_t}$ for the reverse-KL barycenter target defined in \eqref{eq:target-unnorm} with normalization constant $Z_{\alpha_t}$ in \eqref{eq:target}. Consequently, $\Delta_t$ directly measures the influence of $y^*$ on the model's own prediction, rather than a discrepancy between two independent models.

\begin{theorem}[Leakage Attenuation]
\label{thm:leakage}
The effective leakage rate induced by \method{} satisfies:
\begin{equation}
\label{eq:attenuated-leakage}
\mathcal{R}_{\mathrm{leak}}^{\mathrm{DemoPSD}} := \EE_t\!\left[(1-\alpha_t)^2 \|\Delta_t\|^2\right] < \EE_t\!\left[\|\Delta_t\|^2\right] = \mathcal{R}_{\mathrm{leak}},
\end{equation}
where the strict inequality holds whenever $\Pr(\alpha_t > 0) > 0$. Moreover, the attenuation is strongest where leakage risk is greatest: since $\alpha_t$ is monotonically increasing in $d_t$ and $d_t$ correlates positively with $\|\Delta_t\|$ (both measure the divergence between $\pi_T^t$ and $\pi_S^t$), positions with larger privileged deviation tend to receive larger $\alpha_t$ and hence stronger suppression.
\end{theorem}

The key message is that \method{} does not merely reduce the average leakage, it \emph{selectively} attenuates positions that contribute most to leakage. The factor $(1-\alpha_t)$  scales down the privileged deviation $\Delta_t$ in the gradient, and since $\alpha_t$ correlates positively with $\|\Delta_t\|$, the suppression is strongest precisely where it is most needed. Full proof is in Appendix~\ref{app:proof-leakage}.


\begin{theorem}[Exploration Preservation]
\label{thm:bestofboth}
Let $\pi_{\text{target}}^{\alpha_t}(v) \propto \big(\pi_T^t(v, y^*)\big)^{1-\alpha_t}\big(\pi_S^t(v)\big)^{\alpha_t}$ be the reverse-KL barycenter target with the leakage attenuation coefficient $\alpha_t \in [0,\alpha_{\max}]$, and write $\Delta_t(v) = \log \pi_T^t(v, y^*) - \log \pi_S^t(v)$ for the log-ratio. The full-teacher target minimized by SDPO is the special case $\alpha_t = 0$, namely $\pi_T^t$. Assume the privileged signal is positively aligned with the model's own unprivileged prediction, i.e.,
\begin{equation}
\label{eq:cov-condition}
\mathrm{Cov}_{q_{\gamma}^{t}}(\Delta_t,\, \log \pi_S^t) \;\geq\; 0
\end{equation}
under every geometric interpolation $q_{\gamma}^{t} \propto (\pi_T^t)^{\gamma}(\pi_S^t)^{1-\gamma}$, $\gamma \in [0,1]$, between the student and teacher's distributions. Then the \method{} target preserves strictly more entropy than the full-teacher target, with the ordering
\begin{equation}
\label{eq:entropy-ordering}
\Ent(\pi_S^t) \;\geq\; \Ent(\pi_{\text{target}}^{\alpha_t}) \;\geq\; \Ent(\pi_T^t),
\end{equation}
holding with strict inequalities whenever $0 < \alpha_t$ and $\pi_T^t \neq \pi_S^t$. Moreover, the entropy gain $\Ent(\pi_{\text{target}}^{\alpha_t}) - \Ent(\pi_T^t) \ge 0$ over the full-teacher target is non-decreasing in $\alpha_t$: the more the teacher's prediction depends on the privileged $y^*$ (the larger $\alpha_t$), the more exploration capacity \method{} retains relative to SDPO.
\end{theorem}
\begin{proof}[Proof sketch]
We write  the family $q_{\gamma}^{t}(v)$ as  $q_{\gamma}^{t}(v) \propto \pi_S^t(v)\,e^{\gamma\Delta_t(v)}$, which is an exponential family with  parameter $\gamma$ and sufficient statistic $\Delta_t$. Differentiating $\Ent(q_{\gamma}^{t}) = -\EE_{q_{\gamma}^{t}}[\log q_{\gamma}^{t}]$ and using $\frac{d}{d\gamma}\EE_{q_{\gamma}^{t}}[f] = \mathrm{Cov}_{q_{\gamma}^{t}}(f, \Delta_t)$ yields $\frac{d\Ent}{d\gamma} = -\mathrm{Cov}_{q_{\gamma}^{t}}(\Delta_t, \log q_{\gamma}^{t})$. Substituting $\log q_{\gamma}^{t} = \log \pi_S^t + \gamma\Delta_t - \log Z_\gamma$ and expanding gives \eqref{eq:entropy-derivative}. Under condition \eqref{eq:cov-condition}, both terms are non-positive for $\gamma > 0$, so $\Ent(q_{\gamma}^{t})$ is decreasing. Since $1-\alpha_t < 1$, the \method{} target $q_{1-\alpha_t}$ has strictly higher entropy than the OPSD target $q_1$. Full proof is in Appendix~\ref{app:proof-bestofboth}.
\end{proof}

The result follows by tracking the entropy along the geometric path $q_{\gamma}^{t} \propto (\pi_T^t)^{\gamma}(\pi_S^t)^{1-\gamma}$ that connects the unprivileged student distribution ($q_0 = \pi_S^t$) to the full teacher distribution ($q_1 = \pi_T^t$). The learning target of \method{} sits at $q_{1-\alpha_t} = \pi_{\text{target}}^{\alpha_t}$, strictly short of the teacher. Along this path, the entropy satisfies
\begin{equation}
\label{eq:entropy-derivative}
\frac{d\,\Ent(q_{\gamma}^{t})}{d\gamma} \;=\; -\,\gamma\,\mathrm{Var}_{q_{\gamma}^{t}}[\Delta_t] \;-\; \mathrm{Cov}_{q_{\gamma}^{t}}(\Delta_t,\, \log \pi_S^t),
\end{equation}
which separates the entropy change into two terms.
The first term $-\gamma\,\mathrm{Var}_{q_{\gamma}^{t}}[\Delta_t]$ is the intrinsic entropy cost of incorporating the privileged signal: any nonconstant multiplicative shift reduces entropy, and this cost grows with the  $\gamma$.
The second term $-\mathrm{Cov}_{q_{\gamma}^{t}}(\Delta_t, \log \pi_S^t)$ captures the interaction with the model's existing predictions, and condition \eqref{eq:cov-condition} requires their positive correlation: tokens to which the model already assigns high probability receive a larger boost from~$y^*$. This is the natural regime for self-distillation, where the teacher is the \emph{same model} with additional answer information and hence predominantly sharpens existing predictions rather than contradicting them.
Because \method{} halts the interpolation at $\gamma = 1-\alpha_t < 1$ rather than at the full teacher $\gamma = 1$, it never pays the final, steepest portion of this entropy cost; the entropy it saves grows with $\alpha_t$, consistent with the 33--98\% entropy improvements over SDPO observed in Table~\ref{tab:training-dynamics}.

These two results characterize why the reverse-KL barycenter target is suitable for privileged self-distillation. Theorem~\ref{thm:leakage} shows that disagreement-dependent weighting attenuates the contribution of privileged deviations, thereby reducing the pressure to imitate privileged information dependent teacher signals. Theorem~\ref{thm:bestofboth} complements this result by showing that, under the stated covariance condition, the barycenter target remains strictly more entropic than the full privileged-teacher target, preserving exploration where the teacher distribution is strongly shaped by~$y^*$. Combined with the gradient analysis in \S\ref{subsec:loss-gradient}, these results indicate that \method{} retains dense token-level supervision when the teacher and student agree, while down-weighting teacher signals that are likely to reflect privileged information on high-disagreement positions.

\section{Experiments}
\label{sec:experiments}

\begin{table}[t]
\centering
\caption{Main results on SciKnowEval. 
\textbf{Bold} indicates the best method per metric. \method{} consistently outperforms both GRPO and SDPO across all four domains and all metrics.}
\label{tab:main-results}
\setlength{\tabcolsep}{6pt}
\renewcommand{\arraystretch}{1.1}
\begin{tabular}{l|ccc|ccc|ccc}
\toprule
& \multicolumn{3}{c|}{\textbf{mean@16}} & \multicolumn{3}{c|}{\textbf{maj@16}} & \multicolumn{3}{c}{\textbf{best@16}} \\
\cmidrule(lr){2-4}\cmidrule(lr){5-7}\cmidrule(lr){8-10}
\textbf{Domain} & GRPO & SDPO & \method{} & GRPO & SDPO & \method{} & GRPO & SDPO & \method{} \\
\midrule
Biology   & 33.51 & 36.88 & \textbf{39.25} & 34.84 & 38.07 & \textbf{40.64} & 58.36 & 64.04 & \textbf{68.51} \\
Chemistry & 65.83 & 71.70 & \textbf{72.98} & 66.72 & 72.41 & \textbf{73.71} & 80.47 & 85.94 & \textbf{90.05} \\
Material  & 76.32 & 76.13 & \textbf{76.53} & 76.50 & 76.24 & \textbf{76.71} & 80.24 & 81.69 & \textbf{81.79} \\
Physics   & 66.31 & 68.98 & \textbf{71.64} & 70.52 & 71.88 & \textbf{74.24} & 82.59 & 85.51 & \textbf{88.13} \\
\midrule
Average & 60.49 & 63.42 & \textbf{65.10} & 62.14 & 64.65 & \textbf{66.33} & 75.42 & 79.30 & \textbf{82.12} \\
\bottomrule
\end{tabular}%
\vspace{0.2cm}
\end{table}

\begin{table}[t]
\centering
\caption{Out-of-distribution generalization on GPQA Extended. Material science has no GPQA counterpart. Values are taken at the final training stage (mean over the last three evaluations). \method{} remains stable and improves slightly across all three GPQA domains, whereas SDPO degrades substantially over training (Figure~\ref{fig:gpqa-ood}).}
\label{tab:gpqa}
\setlength{\tabcolsep}{8pt}
\begin{tabular}{l|ccc|c}
\toprule
\textbf{Method} & \textbf{Biology} & \textbf{Chemistry} & \textbf{Physics} & \textbf{Average} \\
\midrule
SDPO & 57.81 & 28.62 & 52.99 & 46.47 \\
\method{} & \textbf{61.42} & \textbf{41.75} & \textbf{59.98} & \textbf{54.38} \\
\bottomrule
\end{tabular}
\vspace{0.2cm}
\end{table}

We evaluate \method{} on scientific reasoning benchmarks, comparing against SDPO and GRPO as the primary baselines. The experiments focus on three aspects: in-domain accuracy, training entropy as an empirical indicator of exploration preservation, and out-of-distribution generalization as a proxy for reduced privileged information leakage.

\subsection{Experimental Setup}
\label{subsec:setup}

\vspace{3pt} \noindent\textbf{Base Model.} We use Qwen3-4B-Instruct~\citep{qwen3} as the base model for all experiments. 

\vspace{3pt} \noindent\textbf{Training Data.} We train on SciKnowEval~\citep{sciknoweval}, a multi-domain scientific reasoning benchmark formulated as 4-choice multiple-choice questions. We train and evaluate separately on four domains: biology, chemistry, material science, and physics. 

\vspace{3pt} \noindent\textbf{Evaluation Benchmarks.} We evaluate the performance on the following benchmarks to assess both in-domain accuracy and out-of-domain generalization:
\begin{itemize}[itemsep=1pt,topsep=1pt]
    \item \textbf{SciKnowEval} (in-domain): Domain-matched test sets for each of the four scientific domains.
    \item \textbf{GPQA Extended}~\citep{rein2023gpqa} (out-of-domain): Graduate-level science questions in biology, chemistry, and physics. It is used to assess generalization beyond the training distribution.
\end{itemize}

\vspace{3pt} \noindent\textbf{Evaluation Metrics.} For each test prompt, we sample 16 rollouts and report three complementary metrics that capture different aspects of model quality: 
\begin{itemize}[itemsep=1pt,topsep=1pt]
    \item \textbf{mean@16}: Average accuracy across 16 sampled rollouts. 
    \item \textbf{maj@16}: Accuracy of the majority-voted answer across 16 rollouts. 
    \item \textbf{best@16}: Best accuracy among 16 rollouts. 
\end{itemize}

\vspace{3pt} \noindent\textbf{Baselines.} We compare \method{} against two baselines:
\begin{itemize}[itemsep=1pt,topsep=1pt]
    \item \textbf{GRPO}~\citep{shao2024deepseekmath}: The standard RLVR baseline that estimates group-relative advantages from binary outcome rewards. 
    \item \textbf{SDPO}~\citep{hubotter2026sdpo}: The  on-policy self-distillation baseline. 
\end{itemize}
All three methods use the same codebase, training infrastructure, base model, and training data, differing only in their optimization objectives: GRPO uses outcome-level reward, SDPO uses dense teacher supervision, and \method{} uses disagreement-modulated reverse-KL barycenter targets.

\begin{figure}[t!]
    \centering
    \begin{subfigure}[t]{0.48\textwidth}
        \centering
        \includegraphics[width=\textwidth]{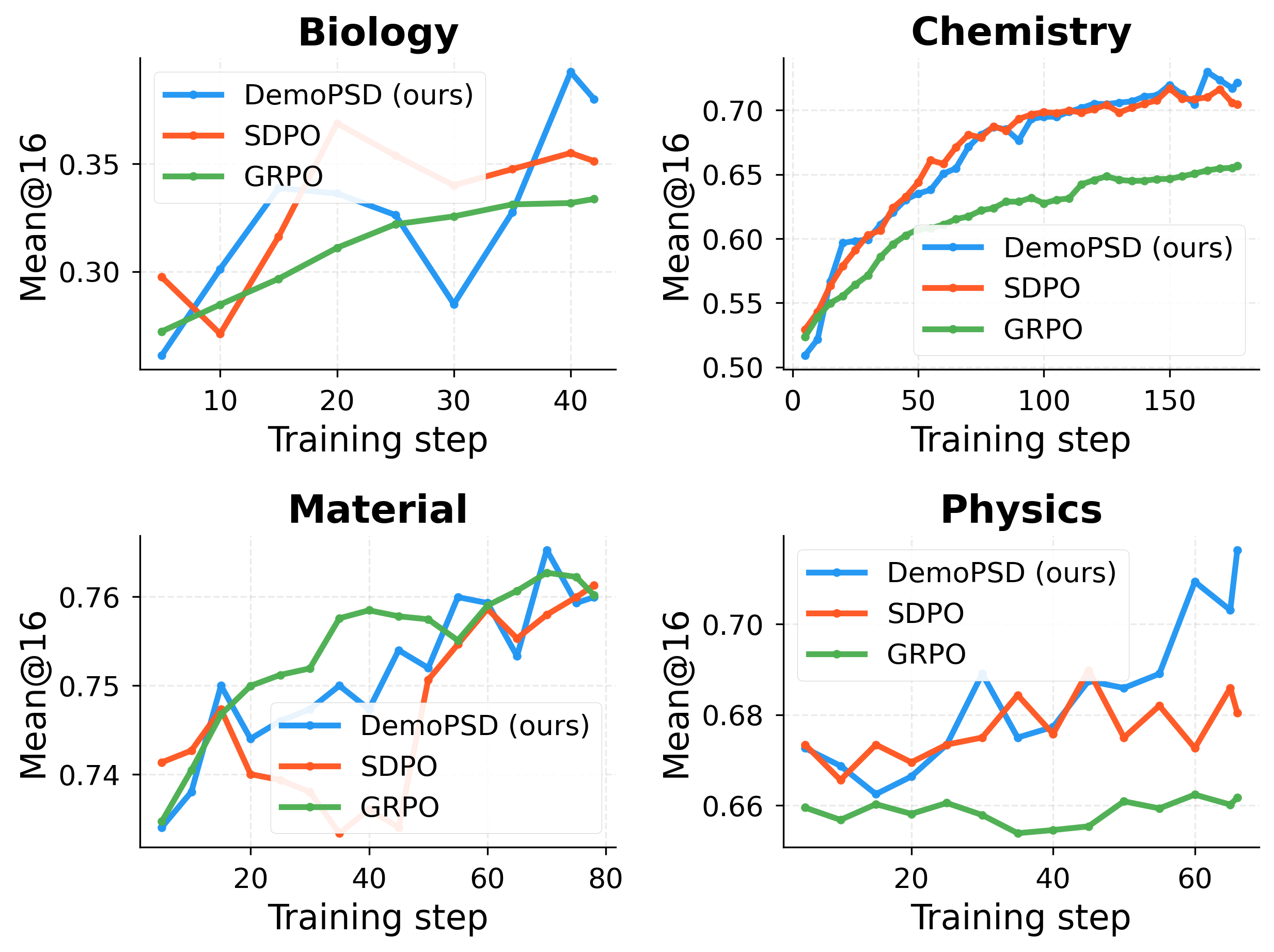}
        \caption{{\small Mean@16 on SciKnowEval over training steps. \method{} maintains higher accuracy than SDPO across training, with the largest margins observed in biology and physics.}}
        \label{fig:accuracy-curves}
    \end{subfigure}
    \hfill
    \begin{subfigure}[t]{0.48\textwidth}
        \centering
        \includegraphics[width=\textwidth]{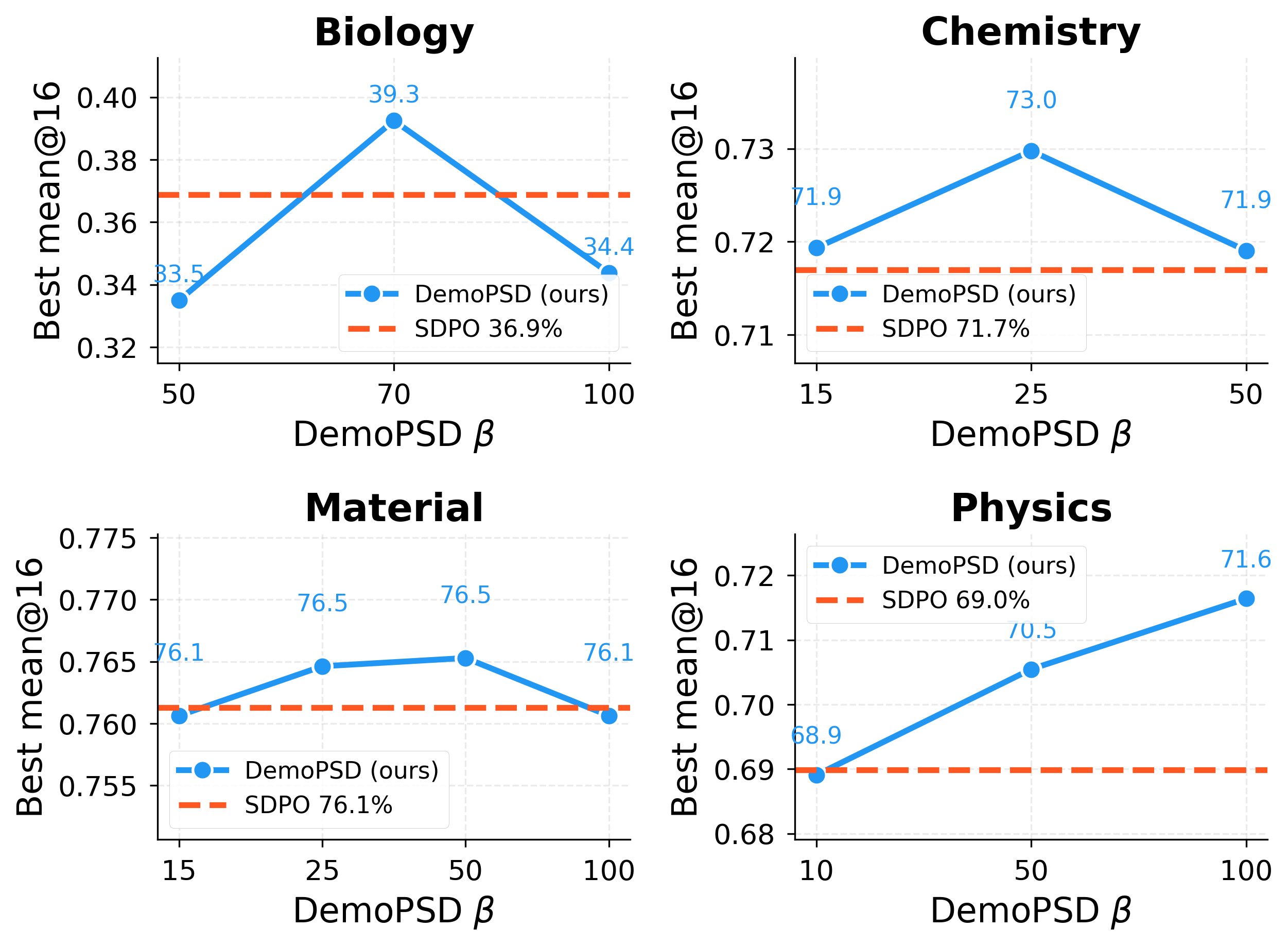}
        \caption{{\small The sensitivity of parameter $\beta$ per domain. The dashed line is the SDPO baseline. \method{} remains competitive with or above the SDPO baseline across $\beta \in [25, 100]$, while the optimal choice of $\beta$ varies by domain.}}
        \label{fig:beta-sensitivity}
    \end{subfigure}
    \vspace{-0.1cm}
    \caption{(a) Test accuracy curves across four domains of SciKnowEval. (b) Sensitivity to $\beta$.}
    \label{fig:acc-and-sensitivity}
\end{figure}

\vspace{3pt} \noindent\textbf{Hyperparameters.} All methods share the following settings: learning rate $1 \times 10^{-6}$, batch size 64, 8 rollouts per prompt for training, max prompt length 2048, max response length 16384, 10 warmup steps, 3 training epochs. For distillation-based methods (SDPO and \method{}), we additionally use top-$k = 100$ for distillation, EMA rate $\eta = 0.05$, and training temperature $= 1.0$  with test temperature $= 0.7$. \method{}-specific parameters: $\alpha_{\max} = 0.15$. The sensitivity parameter $\beta$ is tuned per domain (see \S\ref{subsec:sensitivity}). GRPO uses a KL penalty coefficient $\beta_{\text{KL}} = 0.04$ and clips importance sampling ratios at 2.0. Training uses 8 NVIDIA H20 GPUs with FSDP, vLLM for rollout generation, and flash attention.

\subsection{Main Results}
\label{subsec:main-results}

Table~\ref{tab:main-results} reports the accuracy results across all four scientific domains.
On average, \method{} improves over SDPO by 1.68 on mean@16, 1.68 on maj@16, and 2.82 on best@16. The best@16 improvement is notably larger, indicating that \method{}'s preserved exploration entropy surfaces higher-quality reasoning paths during sampling. Compared to GRPO, the total gain from \method{} is 5.21 on mean@16, demonstrating that the combination of dense supervision and selective adoption leads to substantial improvement. 

Figure~\ref{fig:accuracy-curves} shows how the test accuracy mean@16 changes with training steps. \method{} matches or outperforms SDPO throughout training, and the difference grows in later epochs. This agrees with our theoretical prediction that reducing leakage becomes more helpful as the student moves closer to the teacher's distribution. Figure~\ref{fig:best16-curves} reports the corresponding best@16 accuracy. The improvement is especially clear under best@16 and grows across training, indicating that the higher-entropy policy maintains broader solution coverage.

\subsection{Out-of-Distribution Generalization}
\label{subsec:ood}

We evaluate the model's out-of-distribution generalization capability on GPQA Extended dataset, which contains graduate-level science questions that differ substantially from SciKnowEval in format, difficulty, and question style. Table~\ref{tab:gpqa} reports the accuracy at convergence, and Figure~\ref{fig:gpqa-ood} traces the full GPQA learning curves. 


\begin{figure}[t]
    \centering
    \includegraphics[width=\textwidth]{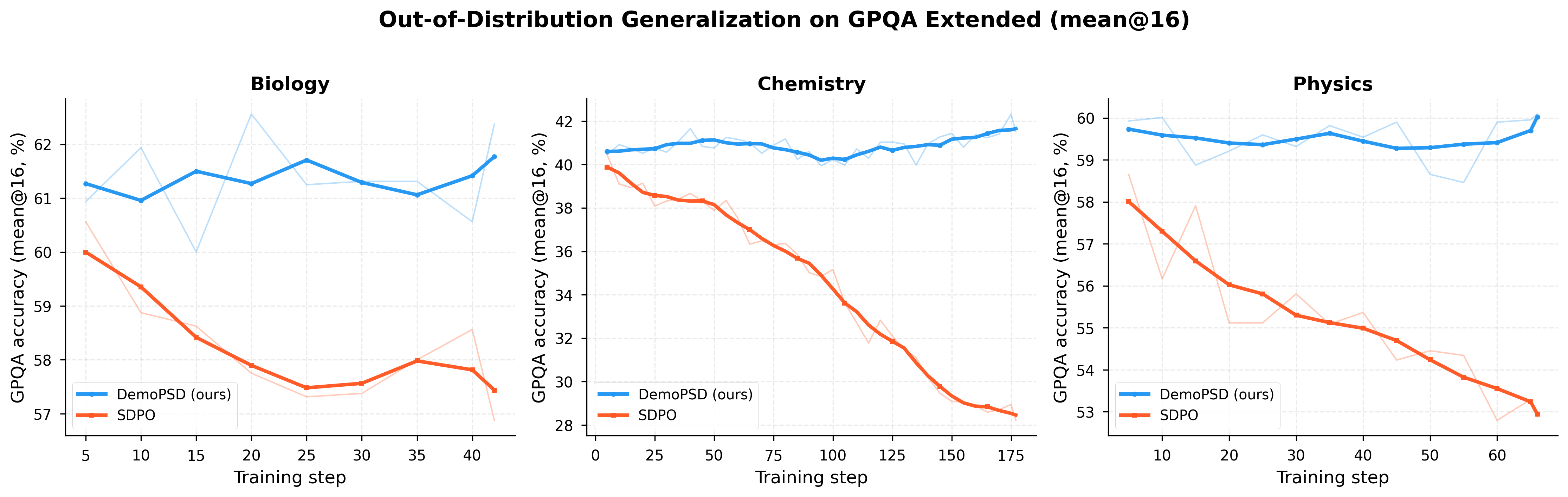}
    \caption{Out-of-distribution generalization on GPQA Extended. Each panel tracks GPQA accuracy over training for one domain (material science has no GPQA counterpart). SDPO reaches its best OOD accuracy early and then degrades as training progresses, consistent with accumulating in-domain overfitting and privileged information leakage. In contrast, \method{} maintains stable OOD performance and achieves an improvement over training.  
    }
    \label{fig:gpqa-ood}
\end{figure}

Although SDPO and \method{} start from comparable OOD accuracy, their performance evolves in substantially different directions over training (Figure~\ref{fig:gpqa-ood}). SDPO reaches its best OOD performance early, but subsequently degrades across all three GPQA domains; the largest drop occurs in chemistry, where accuracy decreases from $40.45$ to $28.62$. This mirrors the in-domain leakage-degradation pattern (\S\ref{subsec:leakage-problem}): by collapsing onto the teacher, SDPO overfits the training distribution and loses the exploratory capacity necessary to transfer to novel questions. In contrast, \method{} maintains stable OOD accuracy throughout training and achieves a measurable improvement, ending $7.91$ above SDPO on average. 

\subsection{Training Dynamics}
\label{subsec:dynamics}

To understand how \method{} achieves its accuracy gains, we examine training dynamics, including entropy, disagreement, and hedging behavior at the final training step (Table~\ref{tab:training-dynamics}).

\begin{table}[t]
\centering
\caption{Training dynamics at the final step. \method{} maintains substantially higher entropy than SDPO across all domains.}
\label{tab:training-dynamics}
\setlength{\tabcolsep}{5pt}
\renewcommand{\arraystretch}{1.05}
\begin{tabular}{l|l|cc|ccc}
\toprule
\textbf{Domain} & \textbf{Method} & \textbf{Entropy} & \textbf{$\Delta$Ent.} & \textbf{mean $\bar{\alpha_t}$} & \textbf{mean $\bar{d_t}$} & \textbf{Active \%} \\
\midrule
\multirow{2}{*}{Biology} & SDPO & 0.602 & -- & -- & -- & -- \\
 & \method{} & \textbf{0.816} & +35.5\% & 0.055 & 0.046 & 64.8 \\
\midrule
\multirow{2}{*}{Chemistry} & SDPO & 0.322 & -- & -- & -- & -- \\
 & \method{} & \textbf{0.555} & +72.4\% & 0.036 & 0.037 & 84.0 \\
\midrule
\multirow{2}{*}{Material} & SDPO & 0.150 & -- & -- & -- & -- \\
 & \method{} & \textbf{0.297} & +98.0\% & 0.033 & 0.031 & 68.8 \\
\midrule
\multirow{2}{*}{Physics} & SDPO & 0.385 & -- & -- & -- & -- \\
 & \method{} & \textbf{0.511} & +32.7\% & 0.040 & 0.026 & 90.6 \\
\bottomrule
\end{tabular}
\vspace{0.2cm}
\end{table}

\vspace{3pt} \noindent\textbf{Entropy Preservation.} \method{} maintains 33--98\% higher final entropy than SDPO across all domains (Figure~\ref{fig:entropy-curves}). The largest entropy gap appears in material science (+98.0\%), where SDPO's entropy drops to 0.150, close to entropy collapse.

\vspace{3pt} \noindent\textbf{Disagreement Sparsity.} The average leakage attenuation coefficient $\bar{\alpha_t}$ stays consistently low (0.033--0.055), while the mean disagreement $\bar{d_t}$ ranges from 0.026 to 0.046. These values indicate that the target remains close to the teacher distribution for most tokens, with strong attenuation applied only to a small subset of positions exhibiting substantial teacher-student disagreement. This pattern is consistent with the selective adoption principle: \method{} preserves the teacher signal on most tokens and applies disagreement-modulated attenuation only at positions where teacher and student's predictions diverge.

\vspace{3pt} \noindent\textbf{Active Sample Fraction.} The ``Active \%'' column in Table~\ref{tab:training-dynamics} denotes the fraction of training samples for which a valid privileged teacher context is available, namely samples whose prompt group contains at least one correct rollout. The active fraction correlates with domain difficulty.

\begin{table}[t]
\centering
\caption{Sensitivity to $\beta$ (mean@16). All configurations use  $\alpha_{\max}=0.15$.}
\label{tab:beta-sensitivity}
\setlength{\tabcolsep}{7pt}
\begin{tabular}{l|cccc}
\toprule
$\boldsymbol{\beta}$ & \textbf{Biology} & \textbf{Chemistry} & \textbf{Material} & \textbf{Physics} \\
\midrule
15 & -- & 71.93 & -- & -- \\
25 & -- & \textbf{72.98} & 76.46 & -- \\
50 & -- & 71.90 & \textbf{76.53} & 70.55 \\
70 & \textbf{39.25} & -- & -- & -- \\
100 & 36.88 & -- & 76.06 & \textbf{71.64} \\
\midrule
SDPO & 36.88 & 71.70 & 76.13 & 68.98 \\
\bottomrule
\end{tabular}

\end{table}

\subsection{Hyperparameter Sensitivity}
\label{subsec:sensitivity}


The key hyperparameter of \method{} is $\beta$, which controls how sharply the leakage attenuation coefficient $\alpha_t$ responds to disagreement. Table~\ref{tab:beta-sensitivity} shows the three best-performing $\beta$ configurations for each domain.



A general pattern emerges that domains where the disagreement is smaller (e.g., physics with mean $\bar{d_t} = 0.026$) benefit from a higher $\beta$ to amplify the weak disagreement signal, while domains with greater disagreement (e.g., biology with mean $\bar{d_t} = 0.046$) benefit from a lower $\beta$ to avoid over-aggressive hedging. Across the range $\beta \in [25, 100]$, \method{} consistently matches or outperforms SDPO, demonstrating moderate robustness.

\vspace{3pt} \noindent\textbf{Remap vs.\ Threshold Mode.} All top-performing configurations adopt the remapped alpha schedule in \eqref{eq:alpha-remap}, which constrains $\alpha_t$ to $[0, \alpha_{\max}]$ and guarantees that the privileged teacher retains at least $(1-\alpha_{\max})$ of the mixture weight. Figure~\ref{fig:beta-sensitivity} further illustrates how accuracy varies with $\beta$ across domains.

\subsection{Disagreement Analysis}
\label{subsec:analysis}

To characterize how disagreement is distributed across tokens, we summarize statistics of the per-token disagreement $d_t$ and leakage attenuation coefficient $\alpha_t$ for the best-performing \method{} run in each domain (Table~\ref{tab:disagreement-stats}).

\begin{table}[t]
\centering
\caption{Statistics of disagreement-based attenuation at the final training step. For most tokens, the target remains close to the teacher distribution, while only a small subset of high-disagreement tokens is interpolated more strongly toward the student distribution.}
\label{tab:disagreement-stats}
\setlength{\tabcolsep}{7pt}
\begin{tabular}{l|cccc}
\toprule
\textbf{Statistic} & \textbf{Biology} & \textbf{Chemistry} & \textbf{Material} & \textbf{Physics} \\
\midrule
Mean leakage
attenuation coefficient $\overline{\alpha}_t$ & 0.055 & 0.036 & 0.033 & 0.040 \\
Mean disagreement $\overline{d}_t$ & 0.046 & 0.037 & 0.031 & 0.026 \\
Active sample fraction & 64.8\% & 84.0\% & 68.8\% & 90.6\% \\
\bottomrule
\end{tabular}
\vspace{-0.2cm}
\end{table}
\begin{figure}[t]
    \centering
    \begin{subfigure}[t]{0.48\textwidth}
        \centering
        \includegraphics[width=\textwidth]{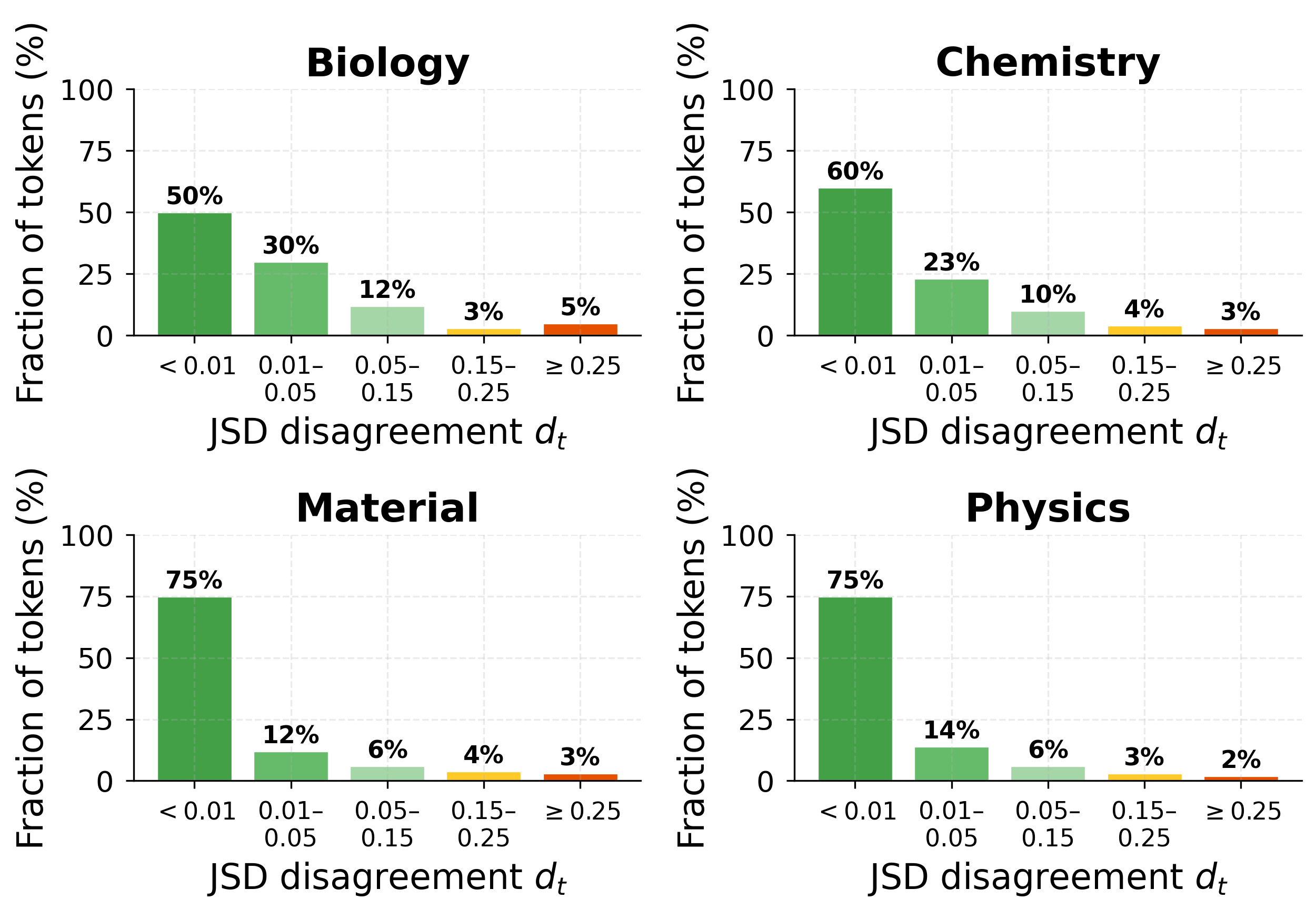}
        \caption{{\small Distribution of per-token JSD disagreement $d_t$. Each panel shows one domain at the final training step. The distribution is heavily right-skewed: the vast majority of tokens have near-zero disagreement, and only 2\%-5\% exceed $ 0.25$.}}
        \label{fig:disagreement-hist}
    \end{subfigure}
    \hfill
    \begin{subfigure}[t]{0.47\textwidth}
        \centering
        \includegraphics[width=\textwidth]{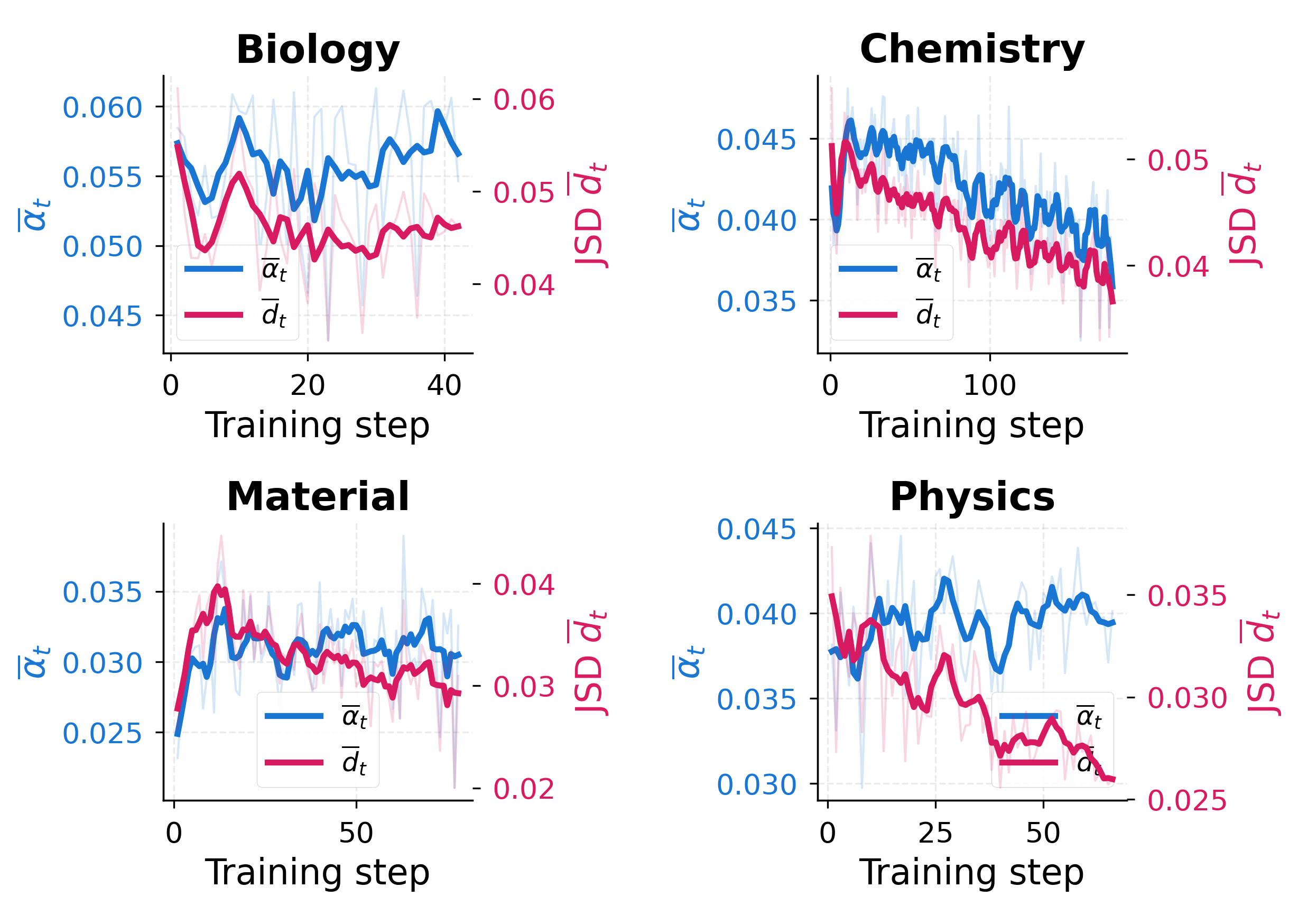}
        \caption{{\small \method{} dynamics over training. Mean leakage attenuation coefficient $\overline{\alpha}_t$ (blue, left axis) and mean JSD disagreement $\overline{d}_t$ (pink, right axis) over training batch per domain. Both quantities remain small and relatively stable.}}
        \label{fig:dupsd-dynamics}
    \end{subfigure}
    \caption{Disagreement analysis of \method{} across four scientific domains.}
    \label{fig:disagreement-analysis}
    \vspace{-0.5cm}
\end{figure}

Across all domains, the disagreement distribution is strongly right-skewed: most tokens exhibit negligible disagreement, allowing the student to remain closely aligned with the teacher, while only a small subset of tokens (typically around 5) shows substantial divergence and triggers stronger attenuation of the privileged teacher signal. This sparsity is beneficial because \method{} preserves the dense token-level supervision of OPSD for the vast majority of positions, while selectively attenuating potential leakage only where the teacher-student mismatch is pronounced.
Figure~\ref{fig:disagreement-hist} visualizes this sparsity directly. Figure~\ref{fig:dupsd-dynamics} shows the evolution of the mean attenuation coefficient $\overline{\alpha}_t$ and  disagreement $\overline{d}_t$ over training steps.

\section{Conclusion}
\label{sec:conclusion}

We introduced \method{}, a self-distillation framework based on selective adoption of teacher guidance: instead of forcing the student to imitate the privileged teacher at every token, \method{} constructs a reverse-KL barycenter target that adaptively balances teacher guidance with the student's own reasoning capacity. Our analysis shows that the disagreement-dependent barycenter weight directly modulates the teacher-induced signal in the training gradient: low-disagreement tokens retain dense teacher supervision, whereas high-disagreement tokens receive attenuated privileged guidance. We formalized this behavior through leakage attenuation and exploration preservation, showing how the proposed learning target reduces pressure to imitate privileged information while maintaining higher-entropy supervision. Empirically, \method{} improves over SDPO and GRPO across four scientific domains, maintains 35--98\% higher training entropy, and generalizes robustly to out-of-distribution benchmarks.


\bibliography{main}

\newpage
\appendix

\section{Detailed Proofs}
\label{app:proofs}

\subsection{Complete Proof of Theorem~\ref{thm:leakage}}
\label{app:proof-leakage}

\begin{proof}
We fix a token position $t$ and suppress the expectations over $x$ and $\hat{y}$ for clarity. Recall the notation: $\pi_\theta^t(v) = \pi_\theta(v \mid x, \hat{y}_{<t})$, $\pi_T^t(v, y^*) = \pi_{\bar{\theta}}(v \mid x, y^*, \hat{y}_{<t})$, $\pi_S^t(v) = \pi_{\bar{\theta}}(v \mid x, \hat{y}_{<t})$, and $\Delta_t(v) = \log \pi_T^t(v, y^*) - \log \pi_S^t(v)$.

Following \citet{yang2026rlsd}, the leakage at position $t$ is driven by the $\nabla_\theta \EE_{\pi_\theta^t}[\Delta_t]$ term, which is the component that carries $y^*$-dependent information and pushes the student to encode privileged correlations. In standard OPSD, this term enters with coefficient $1$, yielding a per-position leakage contribution proportional to $\|\Delta_t\|^2$. In \method{}, the same term enters with coefficient $(1-\alpha_t)$, yielding per-position contribution $(1-\alpha_t)^2 \|\Delta_t\|^2$.

Since $\alpha_t = (\sigma(\beta \cdot d_t) - 0.5) \cdot 2 \cdot \alpha_{\max} > 0$ whenever $d_t > 0$, we have $(1-\alpha_t) < 1$ on all positions with nonzero disagreement. Taking expectations over positions:
\begin{align}
\mathcal{R}_{\mathrm{leak}}^{\mathrm{DemoPSD}} &= \EE_t\!\left[(1-\alpha_t)^2 \|\Delta_t\|^2\right] \nonumber \\
&= \Pr(d_t = 0)\cdot \EE\!\left[\|\Delta_t\|^2 \mid d_t = 0\right] + \Pr(d_t > 0)\cdot \EE\!\left[(1-\alpha_t)^2 \|\Delta_t\|^2 \mid d_t > 0\right] \nonumber \\
&< \Pr(d_t = 0)\cdot \EE\!\left[\|\Delta_t\|^2 \mid d_t = 0\right] + \Pr(d_t > 0)\cdot \EE\!\left[\|\Delta_t\|^2 \mid d_t > 0\right] \nonumber \\
&= \EE_t\!\left[\|\Delta_t\|^2\right] = \mathcal{R}_{\mathrm{leak}},
\end{align}
where the strict inequality uses $(1-\alpha_t)^2 < 1$ on $\{d_t > 0\}$, which has positive probability by assumption. Since $\alpha_t$ is monotonically increasing in $d_t$ and $d_t = \mathrm{JSD}(\pi_S^t \| \pi_T^t)$ correlates with $\|\Delta_t\|$, the attenuation factor $(1-\alpha_t)^2$ is smallest at positions with the largest $\|\Delta_t\|$, concentrating the suppression where it is most needed. \qedhere
\end{proof}

\subsection{Complete Proof of Theorem~\ref{thm:bestofboth}}
\label{app:proof-bestofboth}

\begin{proof}
We fix a token position $t$, using the same notation as in Appendix~\ref{app:proof-leakage}.

\vspace{3pt}
\noindent\textbf{Step 1: Exponential family structure.}
Write $q_{\gamma}^{t}(v) = \pi_S^t(v)\,e^{\gamma\,\Delta_t(v)}/Z_\gamma$ where $Z_\gamma = \sum_v \pi_S^t(v)\,e^{\gamma\,\Delta_t(v)}$ is the partition function and $\Delta_t(v) = \log \pi_T^t(v, y^*) - \log \pi_S^t(v)$. This is a one-parameter exponential family with  parameter $\gamma$, sufficient statistic $\Delta_t(v)$, and base measure $\pi_S^t$.

At the boundary values: $q_0(v) = \pi_S^t(v)$ and $q_1(v) = \pi_S^t(v) e^{\Delta_t(v)}/Z_1 = \pi_T^t(v, y^*)$. The \method{} target corresponds to $\gamma = 1-\alpha_t$, so $q_{1-\alpha_t} = \pi_{\text{target}}^{\alpha_t}$.

Standard exponential family identities give:
\begin{equation}
\frac{d\log Z_\gamma}{d\gamma} = \EE_{q_{\gamma}^{t}}[\Delta_t], \qquad \frac{d^2\log Z_\gamma}{d\gamma^2} = \mathrm{Var}_{q_{\gamma}^{t}}[\Delta_t] \geq 0.
\end{equation}

A key property we will use: for any function $f\colon \mathcal{V} \to \mathbb{R}$,
\begin{equation}
\frac{d}{d\gamma}\EE_{q_{\gamma}^{t}}[f] = \mathrm{Cov}_{q_{\gamma}^{t}}(f,\,\Delta_t). \label{eq:app-exp-cov}
\end{equation}

\vspace{3pt}
\noindent\textbf{Step 2: Entropy derivative.}
Since $\log q_{\gamma}^{t}(v) = \log \pi_S^t(v) + \gamma\,\Delta_t(v) - \log Z_\gamma$, the entropy is:
\begin{equation}
\Ent(q_{\gamma}^{t}) = -\EE_{q_{\gamma}^{t}}[\log q_{\gamma}^{t}] = -\EE_{q_{\gamma}^{t}}[\log \pi_S^t] - \gamma\,\EE_{q_{\gamma}^{t}}[\Delta_t] + \log Z_\gamma.
\end{equation}
Differentiating each term with respect to $\gamma$:
\begin{align}
\frac{d}{d\gamma}\big(-\EE_{q_{\gamma}^{t}}[\log \pi_S^t]\big) &= -\mathrm{Cov}_{q_{\gamma}^{t}}(\log \pi_S^t,\,\Delta_t), \label{eq:app-term1} \\
\frac{d}{d\gamma}\big(-\gamma\,\EE_{q_{\gamma}^{t}}[\Delta_t]\big) &= -\EE_{q_{\gamma}^{t}}[\Delta_t] - \gamma\,\mathrm{Var}_{q_{\gamma}^{t}}[\Delta_t], \label{eq:app-term2} \\
\frac{d\log Z_\gamma}{d\gamma} &= \EE_{q_{\gamma}^{t}}[\Delta_t]. \label{eq:app-term3}
\end{align}
\eqref{eq:app-term1} uses \eqref{eq:app-exp-cov} with $f = \log \pi_S^t$ and \eqref{eq:app-term2} uses the product rule and \eqref{eq:app-exp-cov} with $f = \Delta_t$. Summing \eqref{eq:app-term1}--\eqref{eq:app-term3}, the $\EE_{q_{\gamma}^{t}}[\Delta_t]$ terms cancel such that
\begin{equation}
\frac{d\,\Ent(q_{\gamma}^{t})}{d\gamma} = -\mathrm{Cov}_{q_{\gamma}^{t}}(\Delta_t,\,\log \pi_S^t) - \gamma\,\mathrm{Var}_{q_{\gamma}^{t}}[\Delta_t]. \label{eq:app-entropy-deriv}
\end{equation}

\vspace{3pt}
\noindent\textbf{Step 3: Monotonicity under the covariance condition.}
Under condition \eqref{eq:cov-condition}, both terms in \eqref{eq:app-entropy-deriv} are non-positive for $\gamma > 0$:
\begin{itemize}[leftmargin=*,itemsep=1pt]
\item $-\gamma\,\mathrm{Var}_{q_{\gamma}^{t}}[\Delta_t] \leq 0$, with strict inequality when $\gamma > 0$ and $\Delta_t$ is nonconstant (i.e., $\pi_T^t \neq \pi_S^t$);
\item $-\mathrm{Cov}_{q_{\gamma}^{t}}(\Delta_t,\,\log \pi_S^t) \leq 0$ by the condition.
\end{itemize}
Hence $\frac{d\,\Ent(q_{\gamma}^{t})}{d\gamma} \leq 0$ for all $\gamma \in [0,1]$, with strict inequality on $(0,1]$ when $\pi_T^t \neq \pi_S^t$.

\vspace{3pt}
\noindent\textbf{Step 4: Entropy ordering.}
Since $\Ent(q_{\gamma}^{t})$ is strictly decreasing on $[0,1]$ when $\pi_T^t \neq \pi_S^t$:
\begin{equation}
\Ent(\pi_S^t) = \Ent(q_0) > \Ent(q_{1-\alpha_t}) = \Ent(\pi_{\text{target}}^{\alpha_t}) > \Ent(q_1) = \Ent(\pi_T^t),
\end{equation}
where the strict inequalities require $0 < \alpha_t < 1$ (so that $0 < 1-\alpha_t < 1$, placing the \method{} target strictly between the two endpoints) and $\pi_T^t \neq \pi_S^t$. \qedhere
\end{proof}

\section{Implementation Details}
\label{app:implementation}

\vspace{3pt} \noindent\textbf{Top-$k$ Distillation.} We extract top-$k=100$ tokens from the student's logits, compute both teacher probabilities on this same subset, and aggregate remaining mass into a tail bucket. This reduces memory from $O(|\mathcal{V}|)$ to $O(k)$ per position. The student's top-$k$ indices are shared with both teacher forwards, ensuring all three distributions are index-aligned.

\vspace{3pt} \noindent\textbf{Probability Floor.} All teacher log-probabilities are clamped: $\log p(v) \leftarrow \max(\log p(v), \log 10^{-8})$ to prevent numerical issues in the geometric mixture computation.

\vspace{3pt} \noindent\textbf{Importance Sampling Clip.} To stabilize training across PPO minibatches, we clip the importance sampling ratio: $\rho = \min(\exp(\log \pi_\theta(y_t) - \log \pi_{\theta_{\text{old}}}(y_t)), 2.0)$.


\vspace{3pt} \noindent\textbf{EMA Schedule.} The unprivileged reference uses EMA rate $\eta = 0.05$, updated once after all minibatches complete within a training step.


\vspace{3pt} \noindent\textbf{Masking.} Only response tokens are included in the loss ($\mathcal{T}$ excludes prompt tokens). Samples without a valid reprompt (demopsd\_mask = 0) have their loss contribution zeroed.

\vspace{3pt} \noindent\textbf{Privileged Context Truncation.} When the privileged prompt (question + correct solution + student response) exceeds the maximum reprompt length (10,240 tokens), the demonstration is truncated from the right, preserving the system/user prefix. This is a deliberate departure from SDPO's error-on-overflow behavior, ensuring training proceeds even with long demonstrations.

\end{document}